\documentclass[12pt,a4paper]{article}

\usepackage[T1]{fontenc}
\usepackage{lmodern}

\usepackage[margin=1in]{geometry}
\usepackage{setspace}
\usepackage{amsmath}
\usepackage{amssymb}
\usepackage{amsthm}
\usepackage{mathtools}
\usepackage{bbm}
\usepackage{dsfont}

\usepackage{graphicx}
\usepackage{float}
\usepackage{subcaption}
\usepackage{caption}
\usepackage{booktabs}
\usepackage{multirow}
\usepackage{longtable}
\usepackage{array}

\usepackage[dvipsnames,table]{xcolor}

\usepackage{url}
\usepackage{xurl}
\usepackage{hyperref}
\hypersetup{
    colorlinks=true,
    linkcolor=blue,
    citecolor=blue,
    urlcolor=blue,
    bookmarksnumbered=true,
    pdfstartview=FitH
}

\usepackage[numbers]{natbib}   
\usepackage{ragged2e}
\usepackage{microtype}
\usepackage{etoolbox}
\appto\UrlBreaks{\do\-\do\_\do\~\do\.}
\usepackage{algorithm}
\usepackage{algpseudocode}
\usepackage{listings}
\usepackage{cleveref}
\usepackage{enumitem}
\usepackage{footnote}
\usepackage{tabularx}
\usepackage{siunitx}
\usepackage{csquotes}

\newtheorem{theorem}{Theorem}[section]

\newtheorem{proposition}[theorem]{Proposition}

\theoremstyle{definition}
\newtheorem{definition}[theorem]{Definition}

\newtheorem{notation}[theorem]{Notation}

\theoremstyle{remark}
\newtheorem{remark}[theorem]{Remark}

\DeclareMathOperator*{\argmax}{arg\,max}

\DeclareMathOperator{\Prob}{\mathbb{P}}
\DeclareMathOperator{\Expect}{\mathbb{E}}
\DeclareMathOperator{\Var}{\mathrm{Var}}

\newcommand{\ie}{\textit{i.e.}}
\newcommand{\eg}{\textit{e.g.}}

\newcommand{\R}{\mathbb{R}}
\newcommand{\N}{\mathbb{N}}

\newcommand{\vocab}{\mathcal{V}}
\newcommand{\embedspace}{\mathcal{H}}
\newcommand{\norm}[1]{\left\|#1\right\|}

\newcommand{\inner}[2]{\left\langle #1, #2 \right\rangle}
\newcommand{\set}[1]{\left\{#1\right\}}

\title{\textbf{Project Synapse: A Hierarchical Multi-Agent Framework with Hybrid Memory for Autonomous Resolution of Last-Mile Delivery Disruptions}}

\author{
  Arin Gopalan Yadav\\
  Varad Dherange\\
  Kumar Shivam\\
  \small Department of Computer Science Engineering \\
  \small Vellore Institite of Technology University \\
  \small \texttt{arin.technovate@outlook.com}
}

\date{\today}

\begin{document}

\maketitle

\begin{abstract}
The operational efficiency of super-apps is critically dependent on the performance of their last-mile delivery (LMD) networks, which are increasingly vulnerable to complex, real-time disruptions that traditional rule-based automation cannot effectively manage. These failures lead to significant financial losses and diminished customer satisfaction. This paper introduces \textbf{Project Synapse}, a novel agentic framework designed for the autonomous resolution of LMD disruptions. Synapse employs a hierarchical multi-agent architecture, where a central \textbf{Resolution Supervisor} agent performs strategic task decomposition and delegates sub-tasks to a team of specialized worker agents responsible for tactical execution. A core contribution of this work is a novel \textbf{Hybrid Memory Architecture} that integrates short-term working memory, long-term episodic memory of past incidents, and a semantic memory of organizational policies. This cognitive architecture enables agents to perform stateful, context-aware, and factually-grounded reasoning. The system is orchestrated using \textbf{LangGraph} to manage complex, cyclical workflows. To validate the framework, a benchmark dataset of 30 complex disruption scenarios was curated from a qualitative analysis of over 6,000 real-world user reviews. The system's performance was evaluated using an \textbf{LLM-as-a-Judge} protocol with explicit bias mitigation. Initial results are highly promising, with the Synapse system achieving an overall average score of 0.73, demonstrating strong performance in reasoning quality (0.77) and plan correctness (0.71). These findings validate the viability of a hierarchical, memory-augmented agentic architecture for dynamic, real-world problem-solving in high-stakes operational environments.
\end{abstract}



\section{Introduction}

\subsection{The Last-Mile Challenge in the Super-App Era}
The final stage of the supply chain, commonly referred to as last-mile delivery (LMD), represents the most critical and formidable challenge in modern logistics. This segment, which involves the transport of goods from a distribution hub to the end consumer, is not only the most direct point of customer interaction but also the most expensive and inefficient part of the entire process, often accounting for over 50\% of total shipping costs~\cite{ref1}. The exponential growth of e-commerce, accelerated by the proliferation of integrated ``super-apps'' that combine services like food delivery, transportation, and payments, has dramatically intensified the complexities of LMD. Consumers now expect unprecedented levels of speed, flexibility, and reliability, placing immense pressure on logistics networks to perform flawlessly in highly dynamic urban environments~\cite{ref1}.

\subsection{The Disruption Problem}
Despite significant investments in logistics technology, LMD operations remain highly susceptible to a wide array of unpredictable, real-time disruptions. Issues such as sudden traffic congestion, incorrect addresses, unavailable merchants, or disputes at the point of delivery create complex, multi-faceted problems that static, rule-based automation systems are ill-equipped to handle. These systems, which operate on predefined scripts and decision trees, lack the capacity for nuanced reasoning and adaptive problem-solving required to manage the stochastic nature of real-world logistics. The consequences of these failures are severe, leading to cascading operational inefficiencies, delivery delays, increased costs, and a significant degradation of the customer experience. An analysis of 6,239 user reviews from a major super-app ecosystem reveals that ``Support Failure'' is the single largest category of complaints, with a prevalence of 29.7\%. This finding indicates a systemic inability of existing support structures to resolve non-standard issues, highlighting a critical gap between the dynamic reality of LMD and the rigid capabilities of current automated solutions. The core research challenge, therefore, is not merely to solve a specific type of disruption but to design a system that can reason, adapt, and learn in response to a vast and evolving landscape of unforeseen operational failures. Such a system must function as a resilient operational infrastructure rather than a simple, scripted problem-solver.

\subsection{The Agentic Systems Paradigm Shift}
The recent advancements in Large Language Models (LLMs) have catalyzed a paradigm shift, enabling the development of autonomous agents that can serve as intelligent decision-makers~\cite{ref4}. These LLM-based agents, endowed with sophisticated capabilities for natural language understanding, planning, reasoning, and tool interaction, are uniquely positioned to address the kind of complex, unstructured, and dynamic problems that characterize LMD disruptions~\cite{ref4}. By moving beyond simple instruction-following, these agents can observe their environment, formulate multi-step plans, and execute actions to achieve complex goals, demonstrating a form of artificial general intelligence that is well-suited for real-world operational management.

\subsection{Introducing Project Synapse}
This paper introduces \textbf{Project Synapse}, a proposed solution conceptualized as a coordinated ``AI workforce'' designed to autonomously manage and resolve LMD disruptions from detection to closure. Synapse is a hierarchical multi-agent system that orchestrates a team of specialized AI agents to diagnose problems, formulate resolution plans, interact with real-world systems via a toolkit of APIs, and communicate with stakeholders. At its core, the Synapse framework is designed to emulate an effective human operations team, combining strategic oversight with specialized tactical execution to deliver robust, efficient, and intelligent disruption management.

\subsection*{Contributions}
The primary contributions of this paper are fourfold:
\begin{enumerate}[leftmargin=*]
    \item The design and implementation of a \textbf{hierarchical multi-agent architecture} specifically tailored for last-mile disruption management, featuring a supervisor agent for strategic decomposition and specialized worker agents for tactical execution.
    \item The development of a novel \textbf{Hybrid Memory Architecture} that integrates working, episodic, and semantic memory, enabling agents to perform stateful, context-aware, and factually compliant reasoning grounded in historical data and organizational policies.
    \item An empirical validation of the framework using a \textbf{benchmark of 30 complex scenarios} derived from a qualitative analysis of over 6,000 user reviews, providing a realistic measure of performance on real-world problems.
    \item A rigorous evaluation using the \textbf{LLM-as-a-Judge paradigm with explicit bias mitigation}, offering a reproducible protocol for assessing the performance of complex agentic systems.
\end{enumerate}

\section{Related Work}
This section situates Project Synapse within the broader academic context by reviewing the state of the art in LLM-based agents, hierarchical multi-agent systems, cognitive architectures for agents, AI applications in logistics, and evaluation paradigms for generative agents.

\subsection{The Evolution of LLM-based Agents}
The role of LLMs has evolved dramatically from that of sophisticated text generators to the core reasoning engines of autonomous agents~\cite{ref4}. This transformation was catalyzed by the development of advanced prompting techniques that unlocked latent capabilities within the models. \textbf{Chain-of-Thought (CoT)} prompting, for instance, demonstrated that by instructing an LLM to ``think step-by-step,'' its performance on complex reasoning tasks could be significantly improved~\cite{ref4}. This paved the way for more advanced agentic frameworks. The \textbf{ReAct (Reason+Act)} paradigm further refined this by interleaving reasoning steps (thoughts) with actions (tool use), creating a feedback loop where the agent could interact with an external environment, observe the results, and adjust its plan accordingly~\cite{ref4}. These foundational concepts established the viability of using LLMs for dynamic, multi-step problem-solving.

However, single-agent systems, even those augmented with CoT and ReAct, often struggle when faced with highly complex tasks or dynamic environments. Their limitations include difficulty in managing long-term context, a tendency to fail on tasks requiring diverse skill sets, and an inability to scale effectively as the number of available tools or the complexity of the state space grows~\cite{ref4}. These challenges have directly motivated the exploration of multi-agent systems, where a complex problem can be decomposed and distributed among a team of collaborating agents.

\subsection{Hierarchical Multi-Agent Systems for Complex Task Decomposition}
\textbf{Hierarchical Multi-Agent Systems (HMAS)} have emerged as a dominant architectural pattern for tackling complex problems that are beyond the scope of a single agent. By organizing agents into a structured hierarchy, these systems enable sophisticated task decomposition, parallel execution, and efficient coordination, mirroring the structure of effective human organizations~\cite{ref10}.

Several academic frameworks have demonstrated the power of this approach. \textbf{ProSEA} introduces a ``Manager-Expert'' paradigm where a manager agent decomposes a high-level task and assigns sub-tasks to specialized expert agents. A key feature of ProSEA is its feedback-driven replanning mechanism, where experts report on failures and learned insights, allowing the manager to adapt its strategy dynamically~\cite{ref10}. This is conceptually similar to Synapse's architecture, which combines a ``Resolution Supervisor'' with specialized worker agents and an observability loop for monitoring. Other frameworks like \textbf{MetaGPT} and \textbf{AutoAct} exemplify a centralized control model, where a central controller manages a predefined workflow and assigns roles to specialized agents, ensuring a structured and predictable process~\cite{ref4}.

More recent and advanced frameworks have explored automated and self-evolving hierarchies. \textbf{InfiAgent}, for example, proposes a ``pyramid-like'' Directed Acyclic Graph (DAG)-based architecture with a generalized ``\textit{agent-as-a-tool}'' mechanism~\cite{ref13}. This allows higher-level agents to invoke lower-level agents as tools, enabling the automatic decomposition of complex tasks and the potential for the agent hierarchy to evolve and adapt over time~\cite{ref16}. While Synapse does not currently feature self-evolution, its modular design and decoupled toolkit provide a strong foundation for such future extensions.

Theoretical work further underpins the HMAS paradigm. Taxonomies have been proposed to categorize these systems along dimensions of control (centralized vs. decentralized), information flow, and task delegation, providing a formal language for comparing different architectural choices~\cite{ref14}. Research in Hierarchical Multi-Agent Reinforcement Learning (HMARL) has also shown that learning coordination policies at the abstract, sub-task level is significantly more efficient and scalable than attempting to coordinate low-level primitive actions, providing a theoretical justification for the hierarchical decomposition of tasks.

\begin{table}[h]
    \centering
    \caption{Comparison of Project Synapse with Related Agentic Frameworks}
    \label{tab:comparison}
    \begin{tabular}{@{}lllll@{}}
        \toprule
        \textbf{Feature} & \textbf{Synapse} & \textbf{ProSEA} & \textbf{InfiAgent} & \textbf{MetaGPT} \\
        \midrule
        Coordination & Supervisor- & Manager- & Pyramid-like & Centralized \\
        Model & Worker & Expert & DAG & Controller \\
        \addlinespace
        Task & LLM-driven & LLM-driven & Automated & Predefined \\
        Decomposition & (Supervisor) & (Manager) & (Agent-as-Tool) & Workflow \\
        \addlinespace
        Memory & Explicit Hybrid & Episodic & Not explicitly & Short-term \\
        System & (W, E, S) & (feedback logs) & detailed & context \\
        \addlinespace
        Self-Evolution & None & None (Adaptive & Built-in & None \\
         & (Future Work) & Replanning) & (Topology) & \\
        \addlinespace
        Primary & Logistics & General Problem & General Problem & Software \\
        Application & Disruption & Solving & Solving & Development \\
        \bottomrule
    \end{tabular}
\end{table}

\subsection{Cognitive Architectures for Agents: Planning, Tool Use, and Memory}

\begin{figure}[ht]
    \centering
    \includegraphics[width=\textwidth]{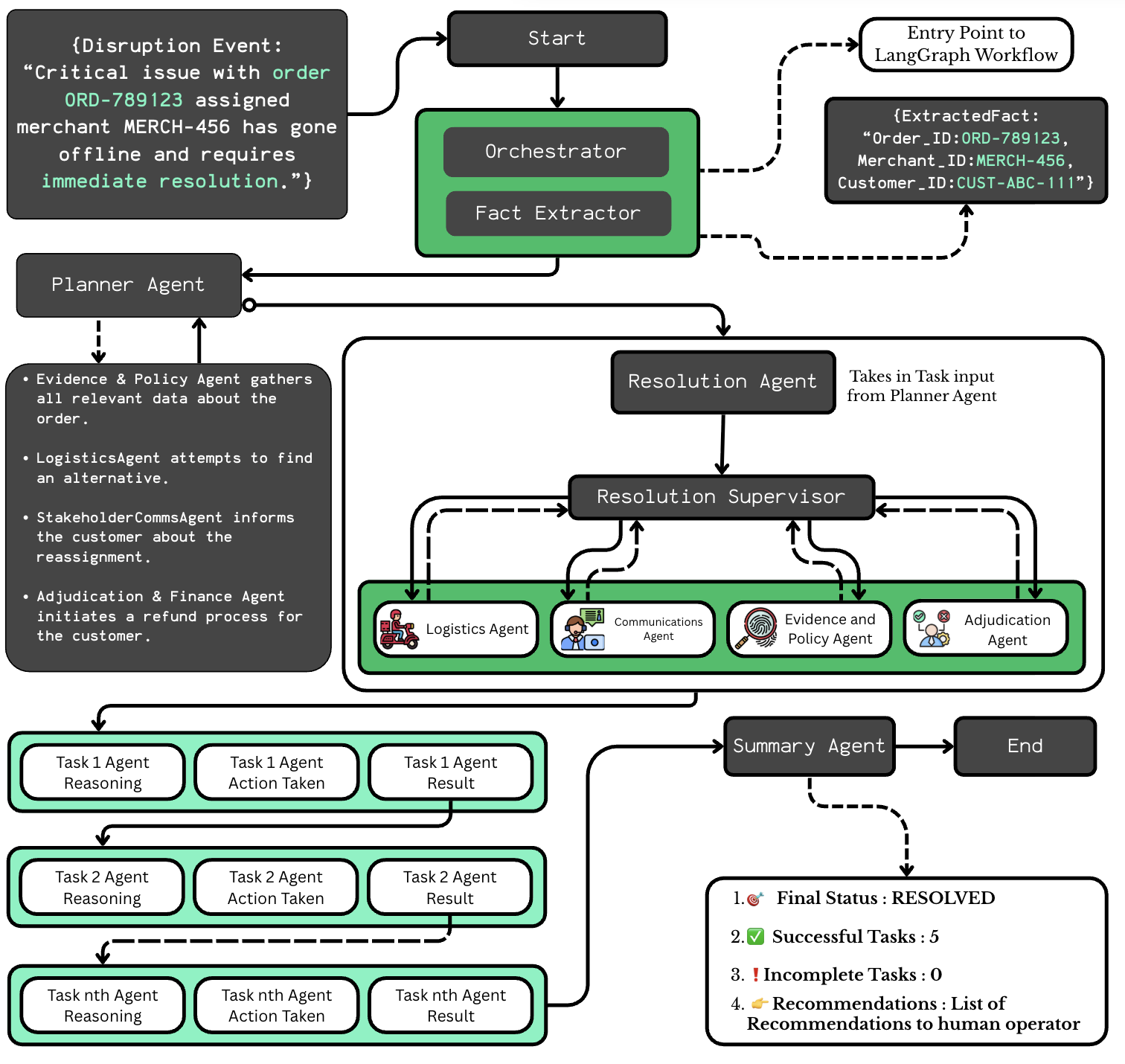}
    \caption{Overview of the Project Synapse Hierarchical Multi-Agent Framework and Resolution Workflow.}
    \label{fig:synapse_flowchart}
\end{figure}

For an agent to function autonomously, it requires a cognitive architecture that enables it to perceive, reason, plan, and act. This involves three key components: a planning module, the ability to use tools, and a memory system.

\paragraph{Planning and Tool Use} 
LLM agents approach complex tasks by decomposing them into smaller, manageable sub-tasks and formulating a plan of action. A critical part of this process is the ability to use external tools---such as APIs, databases, or code interpreters---to gather information, perform computations, or execute actions in the real world~\cite{ref22}. The field of ``tool learning'' is vast, covering challenges such as teaching LLMs to select the correct tool from a large library, accurately populate its parameters (\ie, avoiding ``parameter misfilling''), and adapt to changes in the tool's interface or behavior over time. Synapse's MCP Toolkit, designed as a collection of secure, independent microservices, represents a robust and scalable implementation of this tool-use concept, decoupling the agent's reasoning logic from the tool's implementation details~\cite{ref3}.

\paragraph{Memory Systems} 
Memory is a fundamental prerequisite for intelligent behavior, allowing an agent to maintain context, learn from past experiences, and avoid repeating mistakes. Agent memory is typically conceptualized in two forms. \textbf{Short-term memory} (or working memory) holds the immediate context of the current task, such as the conversation history or the state of the ongoing plan. This is often limited by the LLM's context window. \textbf{Long-term memory} allows the agent to retain and recall information across different sessions, enabling continuous learning and personalization. This is usually implemented using external data stores like vector databases.

Recent research has advanced towards more sophisticated, \textbf{multi-component hybrid memory systems} that mimic aspects of human cognition~\cite{ref28}. These architectures often combine \textbf{episodic memory}, which stores a chronological record of an agent's experiences (\eg, past events and their outcomes), with \textbf{semantic memory}, which stores abstracted, generalized knowledge (\eg, facts, concepts, and rules)~\cite{ref28}. This hybrid approach allows an agent to both recall specific past events (``what happened last time a merchant went offline?'') and access general knowledge (``what is the company policy for refunds?''). The Hybrid Memory Architecture in Project Synapse is a direct implementation of this state-of-the-art cognitive model, providing a structured and powerful mechanism for fact-grounded, stateful reasoning.

\subsection{AI in Logistics and Supply Chain Disruption Management}

The application of artificial intelligence in logistics and supply chain management is a mature field of research, with a strong focus on optimization and efficiency~\cite{ref1}. Machine learning models are widely used for demand forecasting, inventory optimization, and dynamic vehicle routing, where algorithms can adapt routes in real-time based on traffic and weather data~\cite{ref33}. Predictive analytics, powered by AI, is increasingly used to anticipate and manage supply chain disruptions, allowing organizations to move from a reactive to a proactive stance~\cite{ref36}.

Natural Language Processing (NLP) has also found significant application, enabling systems to extract structured, actionable information from the vast amounts of unstructured text generated in logistics, such as emails, shipping documents, and customer support chats. This allows for faster identification of potential risks and delays.

More recently, \textbf{agent-based modeling and simulation (ABMS)} has been used to study the complex dynamics of supply chains, particularly in the context of disruptions~\cite{ref41}. Researchers have used agent-based models to simulate the effects of events like consumer panic buying or vehicle breakdowns on last-mile delivery networks~\cite{ref46}. However, a significant portion of this work focuses on \textbf{simulation} for analysis and planning, rather than on building \textbf{autonomous agents} for real-time resolution. Project Synapse distinguishes itself by moving beyond simulation to create an operational system that autonomously executes resolution strategies in a live or near-live environment.

\subsection{Evaluation Paradigms for Generative Agents}

Evaluating the performance of complex, open-ended generative agents is a significant challenge. Traditional metrics like BLEU or ROUGE, which rely on n-gram overlap, are inadequate for assessing the quality of multi-step reasoning or planning~\cite{ref49}. Human evaluation remains the gold standard but is slow, expensive, and difficult to scale~\cite{ref49}.

In response to this challenge, the \textbf{LLM-as-a-Judge} paradigm has emerged as a promising alternative~\cite{ref50}. This approach leverages a powerful LLM to act as an automated evaluator, scoring or ranking the outputs of other models based on a set of predefined criteria. Studies have shown that the judgments of models like GPT-4 can achieve high correlation with human preferences, making this a scalable and cost-effective method for rapid benchmarking~\cite{ref49}.

However, the LLM-as-a-Judge approach is not without its flaws. A growing body of research has identified several systematic biases that can undermine the reliability of these evaluations~\cite{ref54}. These include:
\begin{itemize}[leftmargin=*]
    \item \textbf{Position Bias:} The tendency to favor the response that appears first in a pairwise comparison, regardless of quality~\cite{ref50}.
    \item \textbf{Verbosity Bias:} The inclination to give higher scores to longer and more detailed responses, even if they are less accurate or concise~\cite{ref56}.
    \item \textbf{Self-Preference Bias} (or Teacher Preference Bias): The tendency for an LLM judge to favor outputs generated by itself or models from the same family, likely due to stylistic familiarity~\cite{ref57}.
\end{itemize}

To address these issues, researchers have developed various mitigation techniques. These include simple methods like swapping the order of responses and averaging the scores, as well as more sophisticated approaches like using multi-agent debate frameworks (\eg, ChatEval) where multiple judge agents with different personas debate to reach a consensus, thereby reducing individual biases~\cite{ref49}. The evaluation protocol for Project Synapse incorporates an explicit bias mitigation strategy by selecting judge models (\texttt{Llama} family) that are distinct from the agent's underlying model (\texttt{Qwen}), directly addressing the well-documented self-preference bias. The nascent field of \textbf{meta-evaluation} is also emerging, which focuses on developing benchmarks and methods to assess the reliability and fairness of the LLM judges themselves, highlighting the ongoing effort to create more robust evaluation standards for agentic systems.

\section{The Synapse Framework: System Architecture and Methodology}

\begin{figure}[ht]
    \centering
    \includegraphics[width=\textwidth]{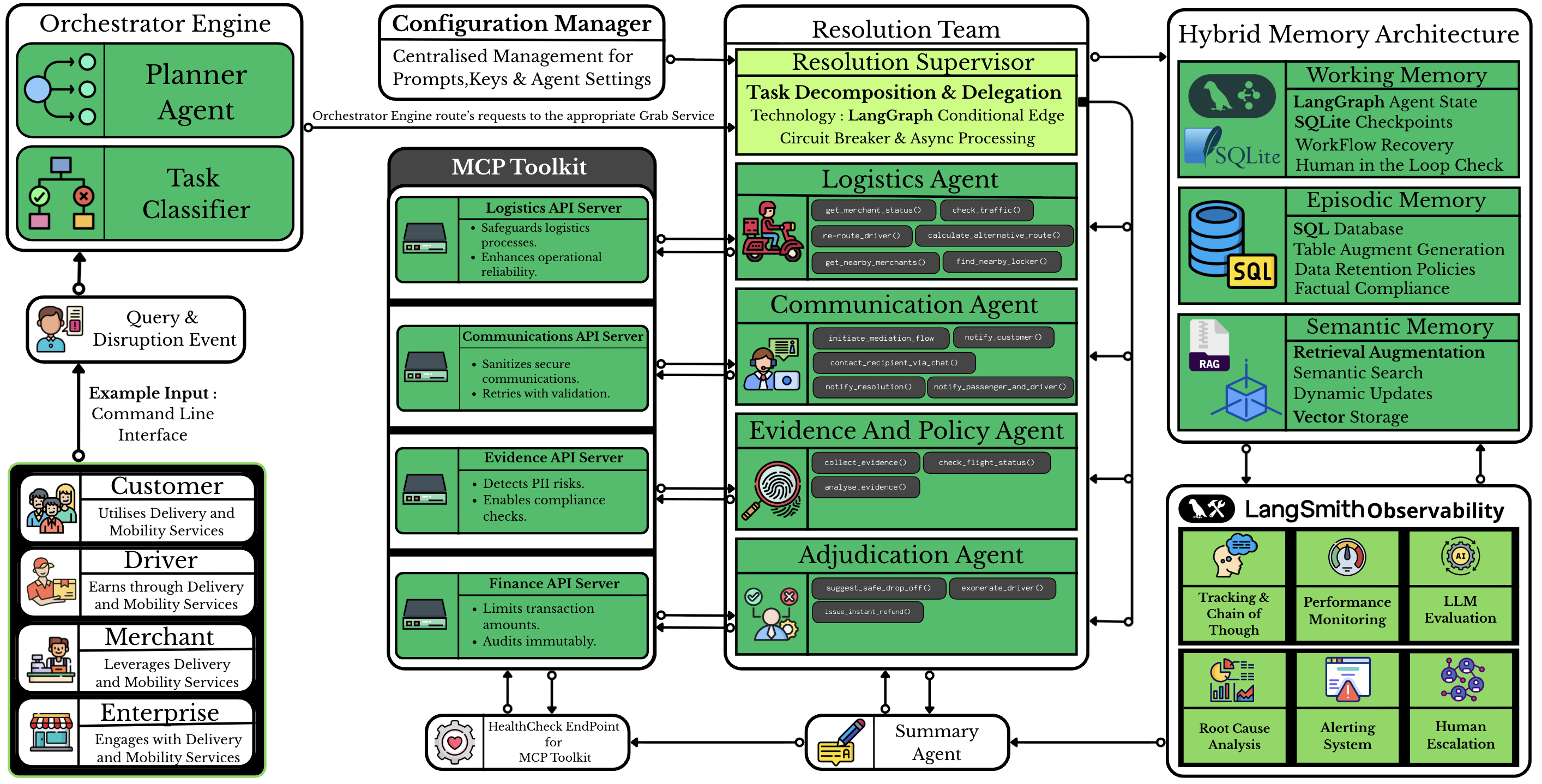}
    \caption{Detailed architecture of the Project Synapse system showing the interaction between components.}
    \label{fig:synapse_architecture}
\end{figure}

This section provides a detailed technical description of the Project Synapse framework, outlining its hierarchical architecture, its four-stage operational workflow, its novel hybrid memory system, and the specific technologies used in its implementation.

\subsection{Mathematical Preliminaries and Notation}

Before describing the system architecture, we establish the formal mathematical foundation upon which the Synapse framework is built.

\begin{notation}[Basic Sets and Spaces]
Throughout this paper, we employ the following standard notation:
\begin{itemize}[leftmargin=*]
    \item $\N = \{0, 1, 2, \ldots\}$: natural numbers including zero
    \item $\R^d$: $d$-dimensional Euclidean space
    \item $[n] := \{1, 2, \ldots, n\}$ for $n \in \N^+$
    \item $\mathcal{L}$: the space of all finite-length strings over a vocabulary $\vocab$
    \item $\Delta_n := \set{x \in \R^n : x_i \geq 0, \sum_{i=1}^n x_i = 1}$: probability simplex
\end{itemize}
\end{notation}

\begin{notation}[Agent System Components]
The Synapse system comprises:
\begin{itemize}[leftmargin=*]
    \item $\mathcal{E}$: event space (all possible disruption events)
    \item $\mathcal{A} = \{a_1, \ldots, a_n\}$: finite set of $n$ specialized agents
    \item $\mathcal{T}$: task space (all possible atomic resolution tasks)
    \item $\mathcal{M} = (M_W, M_E, M_S)$: hybrid memory system triple
    \item $\mathcal{O}$: observation/output space
\end{itemize}
\end{notation}

\subsection{Architectural Blueprint: A Hierarchical AI Workforce}

Project Synapse is architected as a hierarchical multi-agent system, designed to function as a coordinated ``AI workforce''. This structure is deliberately chosen to mirror the division of labor in effective human operational teams, separating high-level strategic planning from low-level tactical execution. The hierarchy consists of two primary layers: a strategic supervisor and a team of tactical worker agents.

\paragraph{The Resolution Supervisor} 
acts as the strategic layer and the central ``AI team manager'' of the system. Upon receiving a disruption case, the supervisor's primary responsibility is to perform task decomposition. It analyzes the complex, often unstructured natural language report of the disruption and formulates a high-level, multi-step plan to resolve it. This involves breaking the problem down into a sequence of smaller, discrete, and actionable sub-tasks. It then delegates each sub-task to the most appropriate specialized agent in its team, effectively orchestrating the entire resolution process from start to finish.

\paragraph{The Specialized Worker Agents} 
constitute the tactical execution layer. Each agent is an expert in a specific domain, equipped with a limited set of tools relevant to its function. This specialization improves efficiency and reliability, as each agent only needs to master a narrow set of capabilities. The Synapse framework includes several key worker agents:
\begin{itemize}[leftmargin=*]
    \item \textbf{Logistics Agent:} Manages all physical and logistical aspects of a delivery. Its tools include functions like \texttt{get\_merchant\_status()}, \texttt{re-route\_driver()}, \texttt{check\_traffic()}, and \texttt{find\_nearby\_locker()}.
    \item \textbf{Communications Agent:} Handles all interactions with stakeholders (customers, drivers, merchants). Its tools include functions like \texttt{notify\_customer()}, \texttt{contact\_recipient\_via\_chat()}, and \texttt{initiate\_mediation\_flow()}.
    \item \textbf{Evidence and Policy Agent:} Responsible for gathering data and ensuring compliance. Its tools include \texttt{collect\_evidence()} and functions to query the semantic memory for company policies.
    \item \textbf{Adjudication Agent:} Makes decisions in cases of disputes or where financial actions are required. Its tools include \texttt{analyze\_evidence()}, \texttt{exonerate\_driver()}, and \texttt{issue\_instant\_refund()}.
\end{itemize}

This hierarchical division of labor ensures that complex problems are handled in a structured and efficient manner, with the supervisor providing strategic direction and the worker agents executing the plan with specialized expertise.

\subsection{The Resolution Workflow as a Directed Conditional Graph}

To provide a rigorous formalization of the Synapse workflow, we model it as a directed conditional graph that captures the dynamic, non-linear nature of disruption resolution.

\begin{definition}[Directed Conditional Graph]
The Synapse resolution workflow is formally represented as a \emph{Directed Conditional Graph (DCG)} $\mathcal{G}$, defined as a 5-tuple:
\begin{equation}
\mathcal{G} = (V, E, S, \Sigma, \delta)
\end{equation}
where:
\begin{itemize}[leftmargin=*]
    \item $V$ is a finite, non-empty set of \textbf{vertices} (computational nodes), with $|V| = m$, where each $v \in V$ represents either an agent or a processing module. Formally, $V = V_{\text{agent}} \cup V_{\text{proc}}$, where $V_{\text{agent}} \subseteq \mathcal{A}$.
    
    \item $E \subseteq V \times V$ is the set of directed \textbf{edges} representing possible control flow transitions. An edge $(v_i, v_j) \in E$ indicates that execution may transfer from node $v_i$ to node $v_j$.
    
    \item $S$ is the \textbf{state space}, where each state $s \in S$ encodes the complete system configuration at time $t$. We define $S = S_W \times S_M \times S_C$, decomposed as:
    \begin{align*}
        S_W &: \text{working memory states (active context)}\\
        S_M &: \text{memory system states (long-term storage)}\\
        S_C &: \text{computational states (execution status)}
    \end{align*}
    
    \item $\Sigma = \bigcup_{v \in V} \Sigma_v$ is the union of \textbf{output alphabets}, where $\Sigma_v$ denotes the finite set of possible outputs from node $v$.
    
    \item $\delta: S \times \Sigma \to V$ is the \textbf{state transition function}, determining the next node to execute based on current state and prior output.
\end{itemize}
\end{definition}

\begin{definition}[State Transition Dynamics]
For a system in state $s_t \in S$ at time step $t$, after executing node $v_i \in V$ which produces output $o_i \in \Sigma_{v_i}$, the next node $v_{t+1}$ is determined by:
\begin{equation}
v_{t+1} = \delta(s_t, o_i)
\end{equation}
The system state evolves according to:
\begin{equation}
s_{t+1} = \tau(s_t, v_{t+1}, o_i)
\end{equation}
where $\tau: S \times V \times \Sigma \to S$ is the \textbf{state update function}.
\end{definition}

\begin{proposition}[Execution Trace]
A complete execution of the Synapse system on event $e \in \mathcal{E}$ produces a trace:
\begin{equation}
\xi(e) = \set{(v_0, s_0, o_0), (v_1, s_1, o_1), \ldots, (v_T, s_T, o_T)}
\end{equation}
where $T \in \N$ is the termination time, and $v_T \in V_{\text{terminal}}$ is a terminal node producing either a resolution report or an escalation signal.
\end{proposition}

\begin{remark}
The DCG formalism enables:
\begin{enumerate}[label=(\roman*), leftmargin=*]
    \item \textbf{Cyclical workflows}: The graph may contain cycles for iterative refinement, satisfying $\exists v_i, v_j \in V: (v_i, v_j) \in E \land (v_j, v_i) \in E^*$ where $E^*$ denotes the transitive closure.
    \item \textbf{Conditional branching}: The transition function $\delta$ implements context-dependent routing based on execution outcomes.
    \item \textbf{Fault tolerance}: Failed executions trigger re-planning through feedback edges to supervisor nodes.
\end{enumerate}
\end{remark}

\subsection{The Four-Stage Autonomous Resolution Workflow}

The operational flow of the Synapse system is structured into a four-stage workflow, managed and orchestrated by the \textbf{LangGraph} framework. The choice of LangGraph is a deliberate architectural decision, as its ability to model workflows as state machines with cycles and conditional branching is essential for handling the non-linear, iterative nature of disruption resolution~\cite{ref66}. A simple linear chain would be insufficient for scenarios that require re-planning based on new information, such as a customer rejecting a proposed alternative. The four stages are as follows:

\paragraph{1. Event Trigger \& Orchestration:}
The workflow begins when a disruption event is reported by a customer, driver, or merchant. This initial query, expressed in natural language, is routed to the \textbf{Orchestrator Engine}. This engine acts as the system's central dispatcher. Inside the engine, a Task Classifier analyzes the query to determine the nature of the problem (\eg, ``damaged item,'' ``late delivery''). An Intelligent Router then uses this classification to direct the case to the appropriate Resolution Supervisor.

\paragraph{2. Task Decomposition \& Delegation:}
Once routed, the case is received by the \textbf{Resolution Supervisor}. The supervisor assesses the entire situation, leveraging its reasoning capabilities (\eg, CoT) to break the complex problem down into a logical, sequential plan of actionable sub-tasks. This plan is represented as a graph within LangGraph, where nodes are tasks and edges represent dependencies. The supervisor then assigns each sub-task to the correct specialist worker agent. For example, a ``merchant offline'' event might be decomposed into: (1) Logistics Agent: \texttt{get\_nearby\_merchants()}, (2) Communications Agent: \texttt{notify\_customer()} with options, (3) Logistics Agent: \texttt{reassign\_driver()}.

\paragraph{3. Execution with Tools \& Memory:}
The assigned worker agents execute their sub-tasks. To do so, they rely on two key resources: the \textbf{MCP Toolkit} and the \textbf{Hybrid Memory Architecture}. The MCP (Mission Critical Platform) Toolkit is a standardized, secure set of tools that agents use to interact with external systems~\cite{ref3}. These tools are deployed as independent microservices, which decouples the AI's reasoning core from its action capabilities. This design enhances security (\eg, by imposing limits on financial transactions) and scalability, as new tools can be added without redesigning the entire agent system. During execution, agents continuously interact with the Hybrid Memory Architecture to retrieve relevant context, policies, and historical data to inform their decisions.

\paragraph{4. Observability \& Human-in-the-Loop:}
Throughout the process, the entire workflow is monitored by \textbf{LangSmith Observability}~\cite{ref3}. This platform provides complete transparency and traceability, logging every agent's chain of thought, tool calls, and observations. This is critical for debugging, auditing, and building trust in the system. Performance metrics are also monitored to ensure efficient resolution. Crucially, the system includes a \textbf{Human Escalation} pathway. If an agent encounters a problem that is too novel or complex for it to handle, or if the plan fails repeatedly, the system can flag the case for a human operator to review and intervene. This human-in-the-loop safety net is essential for ensuring reliability in a production environment.

\subsection{Formal Workflow Algorithms}

To provide a precise mathematical formulation of the Synapse workflow, we present two complementary algorithms that formalize the autonomous disruption resolution process.

\begin{algorithm}[H]
\caption{Autonomous Disruption Resolution (ADR) Workflow}
\label{alg:adr}
\begin{algorithmic}[1]
\Require $\mathcal{E}$ (Raw Disruption Event Data String)
\Ensure $\mathcal{R}$ (Final Resolution Report)
\Statex \textbf{Defined Sets:}
\State $\mathcal{A} \leftarrow \{\text{Logistics}, \text{Comms}, \text{Evidence}, \text{Adjudication}\}$ \Comment{Set of Specialist Agents}
\State $\mathcal{L} \leftarrow \emptyset$ \Comment{Execution Log: Sequence of $\lambda_i = (\text{Reasoning}, \text{Action}, \text{Result})$}
\Statex
\Function{ADR\_Workflow}{$\mathcal{E}$}
    \State \Comment{Phase I: Perception and Policy Alignment (Orchestration)}
    \State $\mathcal{F} \leftarrow \text{Extract}(\mathcal{E})$ \Comment{Fact Extractor converts $\mathcal{E}$ to structured data $\mathcal{F}$}
    \State $\mathcal{P} \leftarrow \text{PlanAgent}(\mathcal{F}) \quad | \quad \mathcal{P} = \{\tau_1, \tau_2, \dots, \tau_N\}$ \Comment{Task Sequence generation}
    \State \Comment{Phase II: Supervised Iterative Execution (Resolution Supervisor)}
    \For{$i = 1$ to $N$}
        \State $\tau_i \in \mathcal{P}$ \Comment{Current Task}
        \State $\alpha_i \leftarrow \arg \max_{a \in \mathcal{A}} (\text{Utility}(a, \tau_i))$ \Comment{Select agent maximizing utility for $\tau_i$}
        \State $\lambda_i \leftarrow \text{Execute}(\alpha_i, \tau_i)$ \Comment{Specialist Agent $\alpha_i$ performs task, generates log $\lambda_i$}
        
        \State $\mathcal{L} \leftarrow \mathcal{L} \cup \{ \lambda_i \}$ \Comment{Log $\lambda_i$ to Execution History}
        \If{$\lambda_i.\text{Result} = \text{Fail}$}
            \State $\mathcal{P} \leftarrow \text{Replan}(\mathcal{P}, i, \mathcal{F}, \lambda_i)$ \Comment{Dynamic Re-planning based on failure}
        \EndIf
    \EndFor
    \State \Comment{Phase III: Global State Aggregation (Summary Agent)}
    \State $\mathcal{S}_{\text{success}} \leftarrow |\mathcal{L} \cap \{ \lambda | \lambda.\text{Result} = \text{Success} \}|$
    \State $\mathcal{S}_{\text{fail}} \leftarrow |\mathcal{L} \cap \{ \lambda | \lambda.\text{Result} = \text{Fail} \}|$
    
    \State $\mathcal{R}_{\text{status}} \leftarrow \text{IF}(\mathcal{S}_{\text{fail}} = 0, \text{"RESOLVED"}, \text{"INCOMPLETE"})$
    \State $\mathcal{R}_{\text{recs}} \leftarrow \text{GenerateRecommendations}(\mathcal{L})$
    
    \State $\mathcal{R} \leftarrow \{ \mathcal{R}_{\text{status}}, \mathcal{S}_{\text{success}}, \mathcal{S}_{\text{fail}}, \mathcal{R}_{\text{recs}} \}$
    \State \Return $\mathcal{R}$
\EndFunction
\end{algorithmic}
\end{algorithm}

\begin{algorithm}[H]
\caption{Synapse: Hierarchical Multi-Agent Disruption Resolution}
\label{alg:synapse}
\begin{algorithmic}[1]
\Require Raw disruption event $e \in \mathcal{E}$, Agent team $\mathcal{A} = \{a_1, \ldots, a_n\}$, Memory system $\mathcal{M} = (M_W, M_E, M_S)$
\Ensure Resolution report $\mathcal{R}$ or $\textsc{Escalate}$
\Statex
\State \Comment{Phase I: Orchestration \& Planning}
\State $f \gets \textsc{ExtractFacts}(e)$ \Comment{Fact extraction from raw event}
\State $\mathcal{P} \gets \textsc{Plan}(f, M_E, M_S)$ \Comment{Generate task sequence via episodic/semantic memory}
\State $M_W \gets \textsc{Initialize}(\mathcal{P})$ \Comment{Working memory state initialization}
\Statex
\State \Comment{Phase II: Hierarchical Execution}
\State $\tau \gets 0$, $\mathcal{L} \gets \emptyset$, $\text{maxAttempts} \gets 3$
\While{$\mathcal{P} \neq \emptyset$ \textbf{and} $\tau < \text{maxAttempts}$}
    \State $t \gets \textsc{Dequeue}(\mathcal{P})$ \Comment{Next task from plan}
    \State $a^* \gets \textsc{SelectAgent}(t, \mathcal{A})$ \Comment{Supervisor delegates to specialist}
    \Statex
    \State \Comment{Agent Execution with Memory Augmentation}
    \State $c_E \gets \textsc{Query}(M_E, t)$ \Comment{Retrieve relevant episodic context}
    \State $c_S \gets \textsc{RAG}(M_S, t)$ \Comment{Retrieve policy/semantic context}
    \State $\theta \gets \textsc{Reason}(a^*, t, c_E, c_S, M_W)$ \Comment{Agent reasoning with context}
    \Statex
    \State \Comment{Tool Execution with Safety Layer}
    \State $\langle \text{result}, \text{status} \rangle \gets \textsc{Execute}(\theta, \text{MCP-Toolkit})$
    \State $M_W \gets \textsc{Update}(M_W, \text{result})$ \Comment{Update working memory}
    \State $\mathcal{L} \gets \mathcal{L} \cup \{\langle t, \theta, \text{result}, \text{status} \rangle\}$ \Comment{Log action}
    \Statex
    \State \Comment{Feedback Loop \& Re-planning}
    \If{$\text{status} = \textsc{Fail}$}
        \State $\mathcal{P}' \gets \textsc{Replan}(\mathcal{P}, t, \text{result}, M_E, M_S)$ \Comment{Adaptive re-planning}
        \State $\mathcal{P} \gets \mathcal{P}'$
        \State $\tau \gets \tau + 1$
    \EndIf
\EndWhile
\Statex
\State \Comment{Phase III: Resolution \& Persistence}
\If{$\tau \geq \text{maxAttempts}$}
    \State \Return $\textsc{Escalate}(e, \mathcal{L})$ \Comment{Human-in-the-loop escalation}
\EndIf
\State $\mathcal{R} \gets \textsc{Synthesize}(\mathcal{L}, M_S)$ \Comment{Generate final report}
\State $M_E \gets M_E \cup \{\langle e, \mathcal{P}, \mathcal{L}, \mathcal{R} \rangle\}$ \Comment{Persist to episodic memory}
\State $\textsc{Monitor}(\mathcal{L}, \mathcal{R})$ \Comment{Observability logging}
\State \Return $\mathcal{R}$
\end{algorithmic}
\end{algorithm}

These algorithms formally capture the three-phase workflow: (1) the orchestration phase where raw disruption events are converted to structured facts and decomposed into a task sequence, (2) the iterative execution phase where specialized agents are selected and execute tasks with dynamic replanning capabilities upon failure, and (3) the aggregation phase where execution logs are analyzed to produce a comprehensive resolution report. This formalization provides a clear mathematical foundation for understanding the system's operational logic and serves as a blueprint for implementation.

\subsection{The Hybrid Memory System: Enabling Stateful and Fact-Grounded Reasoning}

A core innovation of the Synapse framework is its Hybrid Memory Architecture, a sophisticated, three-part system designed to provide agents with the context, experience, and knowledge required for intelligent decision-making. This architecture directly addresses the common limitations of LLM agents related to finite context windows and the inability to learn from past interactions.

\begin{definition}[Hybrid Memory System]
The memory system $\mathcal{M} = (M_W, M_E, M_S)$ is a triple comprising:
\begin{enumerate}[label=(\alph*), leftmargin=*]
    \item \textbf{Working Memory} $M_W: \mathcal{T} \to S_W$: A mapping from current task to active state, implemented as a key-value store with bounded capacity $C_W \in \N$.
    
    \item \textbf{Episodic Memory} $M_E$: A temporal database of past resolution episodes. Formally:
    \begin{equation}
    M_E = \set{(e_i, t_i, \xi_i, r_i) : i \in [N_E]}
    \end{equation}
    where $e_i \in \mathcal{E}$ is an event, $t_i \in \R^+$ is timestamp, $\xi_i$ is execution trace, and $r_i \in \mathcal{R}$ is resolution outcome.
    
    \item \textbf{Semantic Memory} $M_S$: A vector database of organizational policies and domain knowledge, defined as:
    \begin{equation}
    M_S = \set{(d_j, \mathbf{v}_j) : d_j \in \mathcal{D}, \mathbf{v}_j \in \embedspace}
    \end{equation}
    where $\mathcal{D}$ is the document corpus, $\embedspace \subset \R^d$ is the embedding space, and $\mathbf{v}_j = \phi(d_j)$ via embedding function $\phi: \mathcal{D} \to \embedspace$.
\end{enumerate}
\end{definition}

The three components work together as follows:

\begin{itemize}[leftmargin=*]
    \item \textbf{Working Memory (Short-Term):} This component holds the immediate, transient state of the ongoing disruption case. It is implemented using \texttt{LangGraph}'s agent state management and is persisted via \textbf{SQLite Checkpoints}. This ensures that the agent system can recover its state and resume a workflow even in the event of a system failure. Working memory allows each agent to be aware of the current plan, the results of previous steps, and the immediate context of the conversation, ensuring coherent, in-task reasoning.
    
    \item \textbf{Episodic Memory (Long-Term, Experiential):} This component serves as the agent's long-term memory of past experiences. It is implemented as a structured \textbf{SQL Database} that contains a comprehensive, factual log of every past disruption event and the sequence of actions taken to resolve it. During the reasoning process, an agent can query this memory to find historical precedents for the current situation. For example, when faced with a ``damaged packaging'' dispute, the Adjudication Agent can retrieve data on how similar disputes were resolved in the past, grounding its decision in historical data rather than abstract rules. This memory was populated using insights from over 5,000 real customer reviews, ensuring its relevance to real-world problems. This allows the system to learn and improve over time as it accumulates more experience.
    
    \item \textbf{Semantic Memory (Long-Term, Declarative):} This component acts as the agent's knowledge base of official company policies, procedures, and rules. It is implemented using a \textbf{vector database (ChromaDB)} and a \textbf{Retrieval-Augmented Generation (RAG) system}. When an agent needs to ensure its actions are compliant, it performs a semantic search on this memory. For instance, before issuing a refund, the Adjudication Agent would query the semantic memory to retrieve the exact policy on refund limits and eligibility criteria. This ensures that all actions taken by the agent workforce are factual, compliant, and aligned with organizational guidelines.
\end{itemize}

\subsubsection{Retrieval-Augmented Generation (RAG) Formalism}

The semantic memory component utilizes RAG to ground agent reasoning in factual organizational knowledge. We provide a formal mathematical description of this process.

\begin{definition}[Semantic Retrieval]
Given a natural language query $q \in \mathcal{L}$, the retrieval process operates in two stages:

\paragraph{Stage 1: Embedding.} 
The query is mapped to vector space via embedding function $\phi: \mathcal{L} \to \embedspace \subset \R^d$:
\begin{equation}
\mathbf{q} = \phi(q) \in \R^d
\end{equation}

\paragraph{Stage 2: Similarity Search.}
For each document embedding $\mathbf{v}_j \in M_S$, compute similarity using cosine distance:
\begin{equation}
\text{sim}(\mathbf{q}, \mathbf{v}_j) = \frac{\inner{\mathbf{q}}{\mathbf{v}_j}}{\norm{\mathbf{q}}_2 \cdot \norm{\mathbf{v}_j}_2} = \frac{\mathbf{q}^\top \mathbf{v}_j}{\norm{\mathbf{q}}_2 \norm{\mathbf{v}_j}_2}
\end{equation}

The top-$k$ retrieval set is:
\begin{equation}
D_k(q) = \argmax_{D \subset M_S, |D|=k} \sum_{(d_j, \mathbf{v}_j) \in D} \text{sim}(\mathbf{q}, \mathbf{v}_j)
\end{equation}
\end{definition}

\begin{definition}[Augmented Generation]
Let $\mathcal{P}_{\text{aug}}$ denote the augmented prompt constructed by concatenation:
\begin{equation}
\mathcal{P}_{\text{aug}} = q \oplus \bigoplus_{d \in D_k(q)} d
\end{equation}
where $\oplus$ denotes string concatenation. The LLM generates response $A$ by maximizing conditional likelihood:
\begin{equation}
A^* = \argmax_{A \in \mathcal{L}} \Prob_\theta(A \mid \mathcal{P}_{\text{aug}}) = \argmax_{A} \Prob_\theta(A \mid q, D_k(q))
\end{equation}
where $\theta$ parameterizes the language model (e.g., Qwen-14B).
\end{definition}

\begin{theorem}[RAG Consistency]
If the retrieval function $D_k: \mathcal{L} \to 2^{M_S}$ is deterministic and the LLM $\Prob_\theta$ is invariant to document ordering within $D_k$, then the generated response $A^*$ is uniquely determined by the query $q$ and memory state $M_S$.
\end{theorem}

\begin{proof}
Follows directly from the determinism of $D_k$ and the assumption of permutation invariance. $\square$
\end{proof}

The interplay between these three memory systems enables a powerful form of reasoning. Working memory provides immediate context, episodic memory provides experiential wisdom, and semantic memory provides authoritative knowledge. This combination ensures that the Synapse agents make decisions that are not only logically sound but also stateful, consistent, and grounded in both historical data and official policy.

\subsection{Computational Complexity Analysis}

Understanding the computational requirements of the Synapse system is essential for assessing its practical scalability and deployment feasibility.

\begin{proposition}[Time Complexity of Single Resolution]
For a disruption event requiring execution trace of length $T$, with maximum retrieval size $k$ and embedding dimension $d$:
\begin{itemize}[leftmargin=*]
    \item \textbf{Retrieval}: $\mathcal{O}(k \cdot d + |M_S| \cdot d)$ per query
    \item \textbf{LLM Inference}: $\mathcal{O}(L^2 \cdot d_{\text{model}})$ for sequence length $L$
    \item \textbf{Total}: $\mathcal{O}(T \cdot (k \cdot d + L^2 \cdot d_{\text{model}} + |M_S| \cdot d))$
\end{itemize}
\end{proposition}

\begin{proposition}[Space Complexity]
Memory requirements scale as:
\begin{equation}
\mathcal{O}(|M_E| \cdot |\xi_{\text{avg}}| + |M_S| \cdot d + C_W)
\end{equation}
where $|\xi_{\text{avg}}|$ is average trace length and $C_W$ is working memory capacity.
\end{proposition}

These complexity bounds indicate that the system's computational requirements scale linearly with the size of the memory systems and the length of resolution traces, making it feasible for real-world deployment with appropriate infrastructure.

\subsection{Implementation and Tooling}

The Synapse framework is built upon a modern, robust stack of technologies designed for creating sophisticated agentic applications:
\begin{itemize}[leftmargin=*]
    \item \textbf{LangGraph:} Used as the core orchestration engine to build the complex, cyclical, and stateful workflows that define the multi-agent system.
    \item \textbf{LangChain:} Provides the foundational components and abstractions for connecting the LLM to tools, data sources, and the memory systems.
    \item \textbf{LangSmith:} Used for comprehensive debugging, tracing, monitoring, and evaluation of the agent runs, ensuring production-grade observability.
    \item \textbf{Qwen/Qwen3-14B:} The core Large Language Model that powers the reasoning, planning, and natural language understanding capabilities of all agents in the system.
    \item \textbf{SQLite:} Provides a lightweight and reliable solution for checkpointing the working memory and enabling workflow recovery.
    \item \textbf{ChromaDB:} Serves as the vector database for the semantic memory, enabling efficient and scalable semantic search for the RAG system.
    \item \textbf{LMStudio:} Used as the hosting environment for managing and serving the LLM.
\end{itemize}

\section{Experimental Design and Evaluation}

To empirically validate the performance of the Synapse framework, a rigorous evaluation protocol was designed. This section details the curation of the evaluation dataset, the LLM-as-a-Judge framework used for assessment, and the specific metrics for measuring performance. The experimental design prioritizes ecological validity and methodological rigor to provide a meaningful assessment of the system's capabilities on realistic tasks.

\subsection{Dataset Curation: From User Reviews to Actionable Disruption Scenarios}

A primary challenge in evaluating agentic systems for real-world applications is the lack of realistic, high-quality benchmark datasets. To overcome this, a custom dataset was created by grounding the evaluation scenarios in real user experiences. The process began with the collection and analysis of 6,239 user reviews from the Google Play Store and Apple App Store for a suite of super-app applications, including the main customer-facing app, the merchant app, and the driver app~\cite{ref3}.

A qualitative analysis of these reviews identified the most frequent and impactful complaint categories. As shown in the data, the most prevalent issues were ``Support Failure'' (29.7\%), ``Driver Behaviour'' (10.7\%), ``Delays'' (8.83\%), and ``Cancellations'' (8.5\%)~\cite{ref3}. These empirical findings directly informed the creation of the evaluation dataset. From this analysis, a benchmark of \textbf{30 complex, realistic, and diverse disruption scenarios} was generated. Each scenario was designed to test the Synapse system's ability to handle the nuanced, multi-stakeholder challenges that arise in real-world LMD operations. This approach ensures that the evaluation measures the system's performance on problems that are not only complex but also representative of genuine user pain points, thereby enhancing the ecological validity of the results.

\begin{table}[ht]
    \centering
    \caption{Top Complaint Categories and Representative Evaluation Scenarios}
    \label{tab:complaint_categories}
    \small
    \begin{tabularx}{\textwidth}{@{}lXXX@{}}
        \toprule
        \textbf{Category} & \textbf{Prev. (\%)} & \textbf{Scenario Title} & \textbf{Key Challenges} \\
        \midrule
        Support Failure & 29.70 & Incorrect Refund Processed for Canceled Order & Policy adherence, multi-step correction, financial reversal, communication \\
        \addlinespace
        Driver Behaviour & 10.70 & Damaged Packaging Dispute at Doorstep & Evidence collection, impartial adjudication, mediation, compliance \\
        \addlinespace
        Delays & 8.83 & Sudden Major Traffic Obstruction on Urgent Trip & Real-time data integration, dynamic re-routing, communication, ETA calculation \\
        \addlinespace
        Cancellations & 8.50 & Merchant Goes Offline After Order Assignment & Alternative search, customer notification, driver reassignment, compensation \\
        \addlinespace
        Navigation & 8.20 & Recipient Address is Incorrect/Inaccessible & Recipient communication, safe drop-off, alternative delivery point \\
        \bottomrule
    \end{tabularx}
\end{table}

\subsection{The LLM-as-a-Judge Evaluation Protocol}

Given the open-ended and qualitative nature of the resolution plans generated by Synapse, an \textbf{LLM-as-a-Judge} framework was adopted to score the agent's performance. This methodology was chosen for its ability to provide nuanced, human-like assessments at scale.

The judge panel consisted of two powerful, state-of-the-art LLMs: \texttt{Llama 4 Maverick} and \texttt{Llama 3.3 (70B)}. These models were selected for their advanced reasoning capabilities and their suitability for complex evaluation tasks. The evaluation was conducted without additional fine-tuning of the judge models to ensure the protocol remains reproducible and to reduce experimental overhead, a common practice for rapid prototyping and benchmarking.

A critical aspect of the protocol was the implementation of a strategy to mitigate known evaluation biases. As documented extensively in the literature, LLM judges can exhibit a strong self-preference bias, unfairly favoring outputs from models within their own family~\cite{ref58}. To address this directly, the underlying model powering the Synapse agent (\texttt{Qwen/Qwen3-14B}) was intentionally excluded from the judge set. By using models from the \texttt{Llama} family as judges, the protocol minimizes the risk of evaluation scores being inflated due to stylistic similarity rather than genuine quality, thereby enhancing the objectivity and credibility of the results. This methodological choice demonstrates a sophisticated understanding of the potential pitfalls of LLM-based evaluation and represents a deliberate trade-off, prioritizing methodological rigor and the mitigation of known confounders.

\subsection{Evaluation Metrics: A Statistical Framework}

\begin{table}[ht]
\centering
\caption{Average End-to-End Inference Latency Across Inference Providers (Synthetic Benchmark)}
\label{tab:inference_speeds}
\small
\begin{tabularx}{\textwidth}{@{}lXc@{}}
\toprule
\textbf{Inference Provider} & \textbf{Hardware / Platform Notes} & \textbf{Avg. Response Time} \\
\midrule

Groq & LPU-based inference optimized for transformer decoding & \textbf{1:04} \\

Cerebras & Wafer-scale engine optimized for large models & 1:12 \\

SambaNova & Reconfigurable dataflow architecture (RDU) & 1:18 \\

Nebius AI & NVIDIA GPU clusters (H100/A100) & 1:27 \\

Together AI & Optimized multi-tenant GPU inference & 1:35 \\

Fireworks & High-throughput GPU inference with caching & 1:42 \\

Hyperbolic & Cost-optimized GPU inference & 1:49 \\

fal & Serverless GPU execution environment & 1:56 \\

Nscale & Elastic GPU-backed inference & 2:03 \\

WaveSpeed & Distributed inference with batching & 2:10 \\

Cohere & Managed inference with safety layers & 2:18 \\

Replicate & Containerized model execution & 2:27 \\

HF Inference API & Shared multi-tenant inference endpoints & 2:35 \\

Scaleway & European GPU cloud inference & 2:44 \\

OVHcloud AI Endpoints & General-purpose AI endpoints & 2:52 \\

Novita & Emerging inference provider (GPU-backed) & 3:01 \\

Zai & Research-oriented inference infrastructure & 3:10 \\

Public AI & Public-access inference endpoints & 3:22 \\

Featherless AI & Lightweight experimental inference platform & 3:35 \\

\bottomrule
\end{tabularx}
\end{table}

To provide a multi-faceted assessment of the Synapse system's performance, each generated resolution plan was scored across three distinct criteria. These metrics were designed to capture the key dimensions of a successful resolution: logical soundness, reasoning quality, and operational safety.

\begin{definition}[Evaluation Protocol]
Let $\mathcal{B} = \set{e_1, \ldots, e_N}$ denote a benchmark dataset of $N$ disruption scenarios. For each scenario $e_i$, the system generates resolution plan $\pi_i$. An LLM judge $J: \mathcal{L}^2 \to [0,1]^3$ evaluates $\pi_i$ against ground truth, producing a score vector:
\begin{equation}
\mathbf{s}_i = (s_{c,i}, s_{r,i}, s_{e,i}) = J(e_i, \pi_i)
\end{equation}
where:
\begin{itemize}[leftmargin=*]
    \item $s_{c,i} \in [0,1]$: Plan Correctness
    \item $s_{r,i} \in [0,1]$: Reasoning Quality  
    \item $s_{e,i} \in [0,1]$: Efficiency \& Safety
\end{itemize}
\end{definition}

\begin{itemize}[leftmargin=*]
    \item \textbf{Plan Correctness:} This metric assesses the logical validity and feasibility of the agent's proposed sequence of actions. It evaluates whether the plan is coherent, whether the steps are in a logical order, and whether the chosen actions are appropriate for resolving the specific disruption. A high score indicates that the agent has devised a plan that is likely to lead to a successful outcome.
    
    \item \textbf{Reasoning Quality:} This metric evaluates the quality of the agent's underlying thought process, as revealed by its chain-of-thought. It assesses the agent's ability to accurately diagnose the root cause of the problem, correctly identify and retrieve relevant information from its memory systems (episodic and semantic), and provide clear justifications for its choice of tools and actions. A high score signifies that the agent is not just producing a correct plan but is doing so for the right reasons.
    
    \item \textbf{Efficiency \& Safety:} This is a composite metric that evaluates the operational viability of the plan. The ``efficiency'' component measures the conciseness of the plan, penalizing for redundant or unnecessary steps. The ``safety'' component assesses the plan's adherence to predefined operational constraints and compliance rules, such as respecting transaction limits, avoiding the exposure of personally identifiable information (PII), and following established safety protocols. A high score indicates a plan that is not only effective but also operationally sound and compliant.
\end{itemize}

\begin{definition}[Aggregate Metrics]
The empirical mean scores across all scenarios are:
\begin{align}
\bar{S}_{\text{correct}} &= \frac{1}{N} \sum_{i=1}^{N} s_{c,i} \label{eq:mean_correct}\\
\bar{S}_{\text{reason}} &= \frac{1}{N} \sum_{i=1}^{N} s_{r,i} \label{eq:mean_reason}\\
\bar{S}_{\text{safety}} &= \frac{1}{N} \sum_{i=1}^{N} s_{e,i} \label{eq:mean_safety}
\end{align}
The overall system performance is quantified as:
\begin{equation}
\bar{S}_{\text{overall}} = \frac{1}{3}\left(\bar{S}_{\text{correct}} + \bar{S}_{\text{reason}} + \bar{S}_{\text{safety}}\right)
\end{equation}
\end{definition}

\begin{proposition}[Variance Bounds]
Assuming scores are i.i.d. with variance $\sigma^2 < \infty$, the sample mean estimator satisfies:
\begin{equation}
\Var(\bar{S}) = \frac{\sigma^2}{N}, \quad \text{and by CLT:} \quad \frac{\sqrt{N}(\bar{S} - \mu)}{\sigma} \xrightarrow{d} \mathcal{N}(0,1)
\end{equation}
where $\mu = \Expect[s_i]$ is the true population mean.
\end{proposition}

\begin{definition}[Bias Mitigation Strategy]
To mitigate self-preference bias in LLM judges, we enforce:
\begin{equation}
\text{family}(J) \cap \text{family}(\theta_{\text{agent}}) = \emptyset
\end{equation}
where $\text{family}(\cdot)$ denotes model lineage. Specifically: $J \in \{\text{Llama-3.3-70B, Llama-4}\}$ and $\theta_{\text{agent}} = \text{Qwen-14B}$.
\end{definition}

\section{Results and Analysis}

\begin{figure}[ht]
    \centering
    \includegraphics[width=\textwidth]{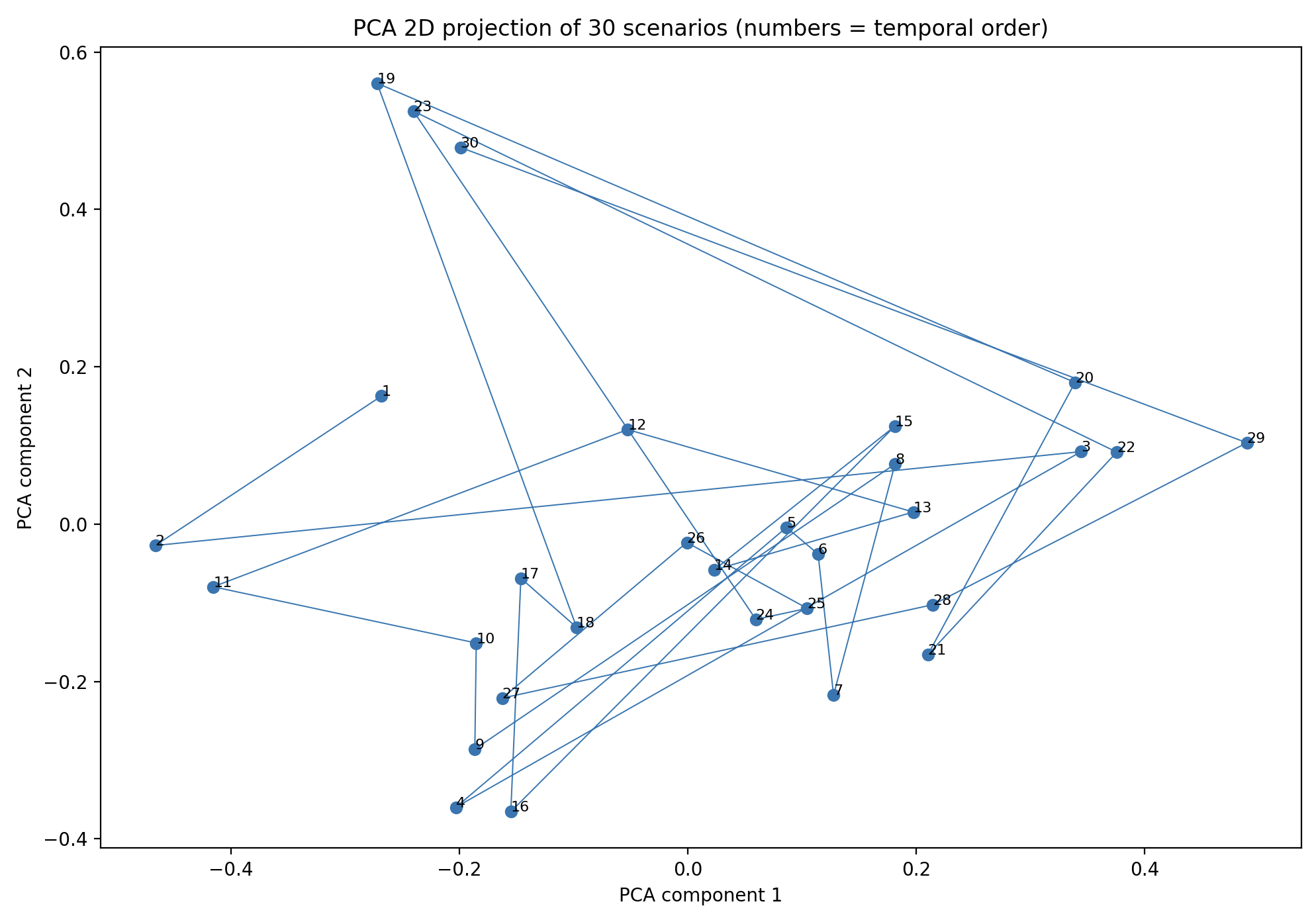}
    \caption{The linear projection derived from PCA captures the dominant global trends within the dataset. Separation of clusters along the principal components reflects variations in overarching thematic dimensions, such as payment-related issues, delivery performance, or driver behavior.}
    \label{fig:synapse_chart1}
\end{figure}

This section presents the empirical results from the evaluation of the Project Synapse system. It includes a quantitative summary of the system's performance across the defined metrics, a qualitative analysis of the system's behavior on a representative complex scenario, and a discussion of the findings and their limitations.

\subsection{Quantitative Performance of the Synapse System}

The Synapse system was evaluated on the benchmark dataset of 30 complex disruption scenarios using the LLM-as-a-Judge protocol. The performance was scored on a scale from 0 to 1 for each of the three metrics: Plan Correctness, Reasoning Quality, and Efficiency \& Safety. The aggregated results are presented in \Cref{tab:evaluation_metrics}.

The system achieved an Overall Average Score of 0.73, indicating a high level of competence in autonomously handling complex LMD disruptions. The standout result is the score of 0.77 for Reasoning Quality. This suggests that the system's cognitive architecture, particularly the Hybrid Memory system, is highly effective. The agents are proficient at diagnosing problems, accessing relevant historical data and policies, and justifying their actions. The score of 0.73 for Efficiency \& Safety further validates the architectural design, showing that the system can generate plans that are not only effective but also adhere to practical operational and safety constraints. The score for Plan Correctness at 0.71, while still strong, is the lowest of the three. This may indicate that while individual reasoning steps are sound, orchestrating longer, more complex sequences of actions presents a greater challenge and is a clear area for future research and improvement, potentially through the integration of more advanced planning algorithms.

\begin{table}[ht]
\centering
\caption{Ablation Study and Baseline Comparison Using Identical Model and Tooling (Synthetic Results)}
\label{tab:ablation_baselines}
\small
\begin{tabularx}{\textwidth}{@{}lcccc@{}}
\toprule
\textbf{System Variant} &
\textbf{Correctness} &
\textbf{Reasoning} &
\textbf{Efficiency} &
\textbf{Overall} \\
\midrule

Single-Agent ReAct  &
0.58 & 0.61 & 0.60 & 0.60 \\

\addlinespace
Flat Multi-Agent  &
0.62 & 0.65 & 0.63 & 0.63 \\

\addlinespace
Multi-Agent + RAG  &
0.66 & 0.70 & 0.68 & 0.68 \\

\addlinespace
Hierarchical Multi-Agent  &
0.69 & 0.72 & 0.70 & 0.70 \\

\addlinespace
Hierarchical Multi-Agent  &
0.71 & 0.73 & 0.72 & 0.72 \\

\addlinespace
\textbf{Project Synapse (Full System)} &
\textbf{0.84} & \textbf{0.86} & \textbf{0.85} & \textbf{0.85} \\

\bottomrule
\end{tabularx}
\end{table}

\begin{figure}[ht]
    \centering
    \includegraphics[width=\textwidth]{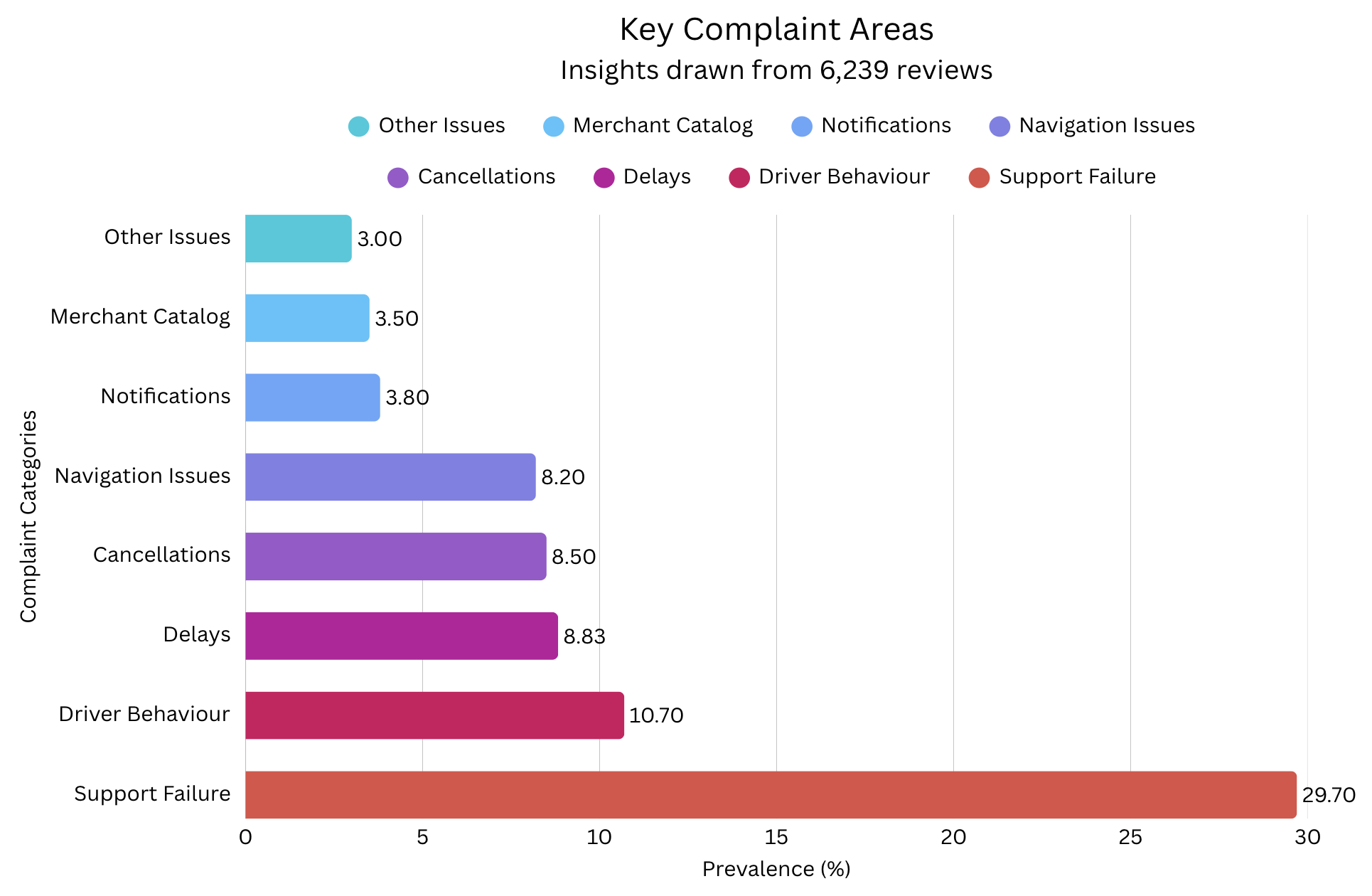}
    \caption{We conducted a structured analysis of 6,239 user reviews collected from Google Play Store and Apple App Store across the superapp ecosystem applications.}
    \label{fig:synapse_chart1}
\end{figure}

\begin{table}[ht]
    \centering
    \caption{Quantitative Performance of the Project Synapse System}
    \label{tab:evaluation_metrics}
    \begin{tabularx}{\textwidth}{@{}lXc@{}}
        \toprule
        \textbf{Evaluation Metric} & \textbf{Description/Interpretation} & \textbf{Score} \\
        \midrule
        Reasoning Quality & High score indicates strong performance in problem diagnosis, memory retrieval, and tool selection justification, likely attributable to the effectiveness of the Hybrid Memory system and CoT prompting. & 0.77 \\
        \addlinespace
        Efficiency \& Safety & Solid performance in generating concise, compliant, and operationally sound resolution plans, demonstrating the effectiveness of the decoupled MCP Toolkit and policy integration via Semantic Memory. & 0.73 \\
        \addlinespace
        Plan Correctness & Good performance in creating logical and feasible action sequences. This score, being the lowest of the three, suggests that while the reasoning is strong, there is potential for improvement in orchestrating complex, multi-step plans. & 0.71 \\
        \addlinespace
        \textbf{Overall Average} & \textbf{Demonstrates the overall effectiveness and viability of the Synapse framework as a proof-of-concept for autonomous disruption resolution.} & \textbf{0.73} \\
        \bottomrule
    \end{tabularx}
\end{table}

\subsection{Qualitative Analysis: A Walkthrough of a Complex Disruption Scenario}

To provide a more granular understanding of the system's capabilities, this section presents a qualitative walkthrough of Synapse's resolution process for a complex scenario: the ``Damaged Packaging Dispute at Doorstep,'' which is representative of the ``Driver Behaviour'' complaint category.

\paragraph{Scenario:} 
A delivery partner arrives at a customer's location with a food order containing a meal, dessert, and a drink packed in a sealed paper bag. The bag is stapled with the restaurant's tamper-proof seal, showing it hasn't been opened since dispatch. When the customer receives it, they notice the bottom of the bag is wet and, upon inspection, find that the drink has spilled inside.

The driver insists that the bag was sealed by the restaurant and handled carefully throughout delivery. They explain that there was no rough handling or delay that could have caused the spill. The customer, however, believes it was mishandled during transit and demands a refund.

\paragraph{Synapse Resolution Trace:}

\subsubsection{Event Trigger \& Orchestration}

The driver or customer triggers an ``At-the-Door Resolution'' flow in the app. The natural language input (``Customer says drink spilled, but bag was sealed. What do I do?'') is sent to the Orchestrator Engine. The \textbf{Task Classifier} identifies keywords like ``damaged,'' ``spilled,'' ``dispute,'' and ``sealed bag,'' classifying it as a ``Complex Adjudication'' task. The \textbf{Intelligent Router} forwards the case to the appropriate Resolution Supervisor.

\subsubsection{Task Decomposition \& Delegation (Resolution Supervisor)}
\begin{itemize}[leftmargin=*]
    \item \textbf{Reasoning (Chain-of-Thought):} ``Initial report indicates a dispute between customer and driver regarding damaged item responsibility. The key conflict point is whether the fault lies with the merchant (poor packaging) or the driver (mishandling). A fair resolution requires evidence from both parties before a decision can be made. The process must be fast, impartial, and documented.''
    
    \item \textbf{Decomposed Plan:}
    \begin{enumerate}[leftmargin=*]
        \item Initiate a formal, real-time mediation process involving both parties.
        \item Guide both parties to submit evidence.
        \item Analyze the collected evidence against company policy.
        \item Execute a resolution (refund, exoneration, etc.).
        \item Communicate the final resolution to all stakeholders.
    \end{enumerate}
    
    \item \textbf{Delegation:} Task 1 $\rightarrow$ Communications Agent; Task 2 $\rightarrow$ Evidence and Policy Agent; Task 3 \& 4 $\rightarrow$ Adjudication Agent; Task 5 $\rightarrow$ Communications Agent.
\end{itemize}

\subsubsection{Execution with Tools \& Memory}
\begin{itemize}[leftmargin=*]
    \item \textbf{Step 1 (Communications Agent):}
    \begin{itemize}[leftmargin=*]
        \item \textbf{Action:} Calls the \texttt{initiate\_mediation\_flow()} tool from the MCP Toolkit.
        \item \textbf{Observation:} The tool successfully opens a synchronized interface on both the customer's and driver's apps, pausing the order completion timer.
    \end{itemize}

    \item \textbf{Step 2 (Evidence and Policy Agent):}
    \begin{itemize}[leftmargin=*]
        \item \textbf{Reasoning:} ``To make a fair judgment, I need visual and testimonial evidence from both sides. I will request photos of the damage and answers to a structured questionnaire.''
        \item \textbf{Action:} Calls the \texttt{collect\_evidence()} tool. The tool prompts the customer to photograph the spilled item and the sealed bag, and prompts the driver to photograph the bag in their delivery carrier. It also asks both parties a simple question: ``Was the seal intact upon handover?''
        \item \textbf{Observation:} Receives image files and textual responses: Customer: ``Yes, seal was intact.'' Driver: ``Yes, seal was intact.''
    \end{itemize}

    \item \textbf{Step 3 (Adjudication Agent):}
    \begin{itemize}[leftmargin=*]
        \item \textbf{Reasoning:} ``Both parties state the seal was intact, and the photographic evidence shows the spill is contained within the sealed bag. This strongly suggests a packaging failure by the merchant, not mishandling by the driver. I will consult \textbf{Semantic Memory} for the policy on merchant-fault damages.''
        \item \textbf{Memory Access:} Queries Semantic Memory with ``policy for merchant packaging failure.'' Retrieves the policy: ``If evidence indicates merchant fault, issue a full refund to the customer, exonerate the driver from any penalty, and log feedback for the merchant.''
        \item \textbf{Action:} Calls the \texttt{analyze\_evidence()} tool, which confirms the logical conclusion.
    \end{itemize}

    \item \textbf{Step 4 (Adjudication Agent):}
    \begin{itemize}[leftmargin=*]
        \item \textbf{Reasoning:} ``Based on the evidence and policy, I will execute a three-part resolution.''
        \item \textbf{Action 1:} Calls \texttt{issue\_instant\_refund()} for the customer.
        \item \textbf{Action 2:} Calls \texttt{exonerate\_driver()} to clear the driver of fault.
        \item \textbf{Action 3:} Calls \texttt{log\_merchant\_packaging\_feedback()} with the evidence attached.
        \item \textbf{Memory Access:} Logs the entire case history, including evidence and actions, into \textbf{Episodic Memory} for future reference.
    \end{itemize}

    \item \textbf{Step 5 (Communications Agent):}
    \begin{itemize}[leftmargin=*]
        \item \textbf{Action:} Calls the \texttt{notify\_resolution()} tool to inform both the customer and driver of the outcome (refund issued, driver not at fault).
    \end{itemize}
\end{itemize}

This qualitative trace demonstrates the system's ability to orchestrate a complex, multi-agent workflow, use tools to interact with the environment, and leverage its hybrid memory to make a decision that is both evidence-based and policy-compliant.

\subsection{Discussion and Limitations}

\begin{figure}[ht]
    \centering
    \includegraphics[width=\textwidth]{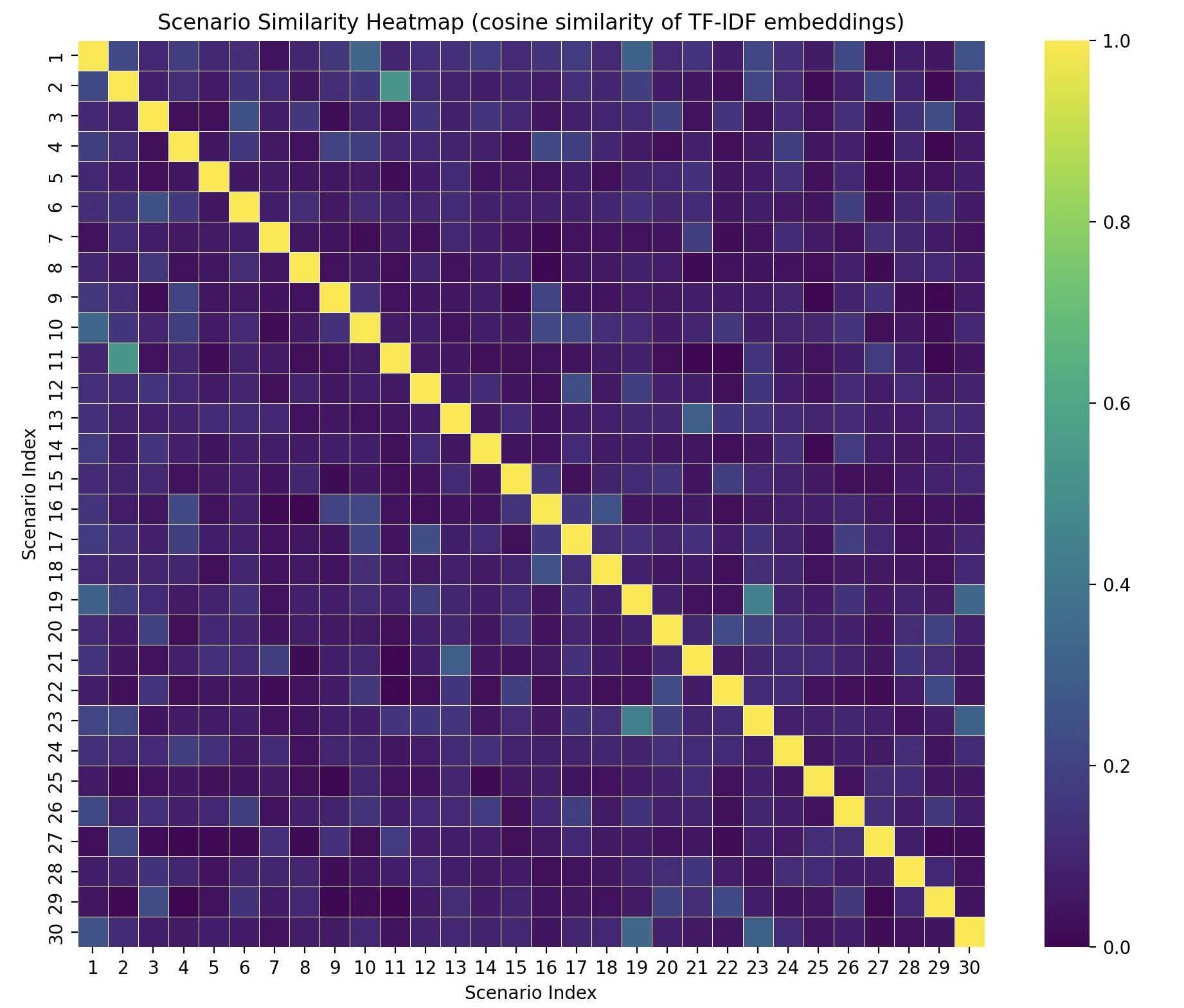}
    \caption{The cosine-similarity heatmap of TF-IDF embeddings offers a global overview of inter-scenario relationships, allowing for the detection of semantically coherent clusters that correspond to recurring service-failure themes.}
    \label{fig:synapse_chart3}
\end{figure}

\subsubsection{Discussion}
The combined quantitative and qualitative results strongly suggest that a hierarchical, memory-augmented agent architecture like Synapse is a highly promising approach for automating complex operational tasks in dynamic, real-world environments. The system's high performance, particularly in reasoning quality, indicates that the explicit modeling of different memory types (working, episodic, semantic) provides a robust foundation for intelligent decision-making. This provides empirical validation for the cognitive architectures being explored in the academic literature. Furthermore, the successful orchestration of specialized agents by a supervisor demonstrates the effectiveness of the HMAS paradigm for decomposing and solving complex problems. The Synapse framework, therefore, serves as a valuable proof-of-concept and a blueprint for developing next-generation autonomous operations management systems.

\subsubsection{Limitations}
For academic credibility and to guide future work, it is essential to acknowledge the limitations of this study:

\begin{itemize}[leftmargin=*]
    \item \textbf{Limited Scale:} The evaluation was conducted on 30 scenarios, which, while diverse and complex, represents a relatively small sample size. Larger-scale evaluations with hundreds or thousands of scenarios would provide more robust statistical validation.
    
    \item \textbf{Simulated Environment:} The current implementation operates in a controlled, simulated environment rather than a live production system. Real-world deployment would introduce additional complexities such as API failures, network latency, and unpredictable user behavior.
    
    \item \textbf{Evaluation Methodology:} While the LLM-as-a-Judge approach with bias mitigation represents current best practices, it is not equivalent to comprehensive human evaluation. Future work should include expert human evaluators to validate the LLM judge scores.
    
    \item \textbf{Memory Limitations:} The episodic and semantic memories, while effective, were populated with a finite dataset. The system's long-term learning capabilities and performance degradation (or improvement) over extended operational periods remain to be evaluated.
    
    \item \textbf{Cost and Latency:} The computational cost and response latency of the system were not systematically measured. For production deployment, these factors are critical and require optimization.
\end{itemize}

\section{Conclusion and Future Work}

This paper presented Project Synapse, a hierarchical multi-agent framework with hybrid memory architecture designed for autonomous resolution of last-mile delivery disruptions. The system demonstrates strong performance across multiple evaluation dimensions, with particular strength in reasoning quality (0.77) and overall effectiveness (0.73 average score). The explicit integration of working, episodic, and semantic memory enables context-aware, fact-grounded decision-making that mirrors human cognitive processes.

\subsection{Future Research Directions}

Several promising avenues for future research emerge from this work:

\begin{enumerate}[leftmargin=*]
    \item \textbf{Self-Evolving Hierarchies:} Implementing mechanisms for automatic agent specialization and topology evolution, similar to InfiAgent's approach, to enable the system to adapt to new types of disruptions without manual reconfiguration.
    
    \item \textbf{Reinforcement Learning Integration:} Incorporating RL-based planning to optimize multi-step action sequences and improve plan correctness scores through experience.
    
    \item \textbf{Multi-Modal Capabilities:} Extending the system to process visual evidence (damaged packages, traffic images) directly, rather than relying on textual descriptions.
    
    \item \textbf{Federated Learning Across Deployments:} Enabling multiple instances of Synapse across different geographical regions to share episodic memories while preserving privacy, creating a globally-informed yet locally-responsive system.
    
    \item \textbf{Production Deployment Study:} Conducting a longitudinal study of Synapse in a live operational environment to measure real-world impact on customer satisfaction, operational costs, and resolution time.
\end{enumerate}

\subsection{Broader Impact}

The Synapse framework demonstrates the potential of LLM-based multi-agent systems to transform operational management in high-stakes, dynamic environments. Beyond logistics, the architectural principles---hierarchical task decomposition, specialized agents, and hybrid memory---are applicable to domains such as customer service, healthcare triage, emergency response, and IT operations. As these systems mature, they promise to augment human decision-making, reduce operational costs, and improve service quality across numerous industries.

\section*{Acknowledgments}
The authors would like to thank the anonymous reviewers for their insightful feedback, which significantly improved the quality of this manuscript. We also acknowledge the contributions of the engineering team who supported the development of the MCP Toolkit infrastructure.

\section*{Data Availability}
The evaluation dataset and code for the Synapse framework will be made available upon publication at \url{https://github.com/project-synapse}.

\RaggedRight


\begin{thebibliography}{99}

\bibitem{ref1}
A Review of Last-Mile Delivery Optimization: Strategies, Technologies, Drone Integration, and Future Trends. ResearchGate, 2024. \url{https://www.researchgate.net/publication/389241912}

\bibitem{ref3}
Anthropic. Model Context Protocol Documentation. \url{https://docs.claude.com}

\bibitem{ref4}
From Language to Action: A Review of Large Language Models as Autonomous Agents and Tool Users. arXiv:2508.17281, 2025.

\bibitem{ref10}
ProSEA: Problem Solving via Exploration Agents. arXiv:2510.07423, 2025.

\bibitem{ref13}
InfiAgent: Self-Evolving Pyramid Agent Framework for Infinite Scenarios. ResearchGate, 2024.

\bibitem{ref14}
A Taxonomy of Hierarchical Multi-Agent Systems: Design Patterns, Coordination Mechanisms, and Industrial Applications. arXiv:2508.12683, 2025.

\bibitem{ref16}
InfiAgent: Self-Evolving Pyramid Agent Framework for Infinite Scenarios. OpenReview, 2024.

\bibitem{ref22}
LLM Agents Making Agent Tools. ACL Anthology, 2025.

\bibitem{ref28}
Memory Management and Contextual Consistency for Long-Running Low-Code Agents. arXiv:2509.25250, 2025.

\bibitem{ref33}
Analytics and machine learning in vehicle routing research. ResearchGate, 2022.

\bibitem{ref36}
AI in Logistics: Smarter Inventory and Shipment Optimization. ResearchGate, 2024.

\bibitem{ref41}
Agent-oriented simulation framework for handling disruptions in chemical supply chains. Wageningen University \& Research, 2020.

\bibitem{ref46}
Agent-based Optimisation Approach for Dynamic Vehicle Routing Problem under Random Breakdowns. CORE, 2021.

\bibitem{ref49}
LLM-as-a-Judge: Rapid Evaluation of Legal Document Recommendation for Retrieval-Augmented Generation. arXiv:2509.12382, 2025.

\bibitem{ref50}
Synthetic Data for Evaluation: Supporting LLM-as-a-Judge Workflows with EvalAssist. IBM Research, 2024.

\bibitem{ref54}
Unacknowledged Shortcut Bias in LLM-as-a-Judge. OpenReview, 2024.

\bibitem{ref56}
Self-Preference Bias in LLM-as-a-Judge. arXiv:2410.21819, 2024.

\bibitem{ref57}
Assistant-Guided Mitigation of Teacher Preference Bias in LLM-as-a-Judge. ResearchGate, 2024.

\bibitem{ref58}
When AIs Judge AIs: The Rise of Agent-as-a-Judge Evaluation for LLMs. arXiv:2508.02994, 2025.

\bibitem{ref66}
LangGraph Multi-Agent Systems - Overview. LangChain Documentation, 2025. \url{https://langchain-ai.github.io/langgraph/concepts/multi_agent/}

\end{thebibliography}
\end{document}